\documentclass[11pt]{article} 
\usepackage[utf8]{inputenc} 
\usepackage[T1]{fontenc}    
\usepackage{hyperref}       
\usepackage{url}            
\usepackage{booktabs}       
\usepackage{amsfonts}       
\usepackage{nicefrac}       
\usepackage{microtype}      
\usepackage{amsmath} 
\usepackage{graphicx}
\usepackage{amsthm}
\usepackage{multirow}
\usepackage{caption,subfigure}
\usepackage{algorithm}
\usepackage{algorithmic}

\def\Re{\mathbb{R}}

\def\A0{{\sf C}}

\def\tA0{\tilde{\sf C}}

\def\argmin{\mathop{\text{\rm arg\,min}}}

\usepackage{eufrak}

\def\Sec#1{Sec.~\ref{#1}}

\def\notes#1{\marginpar{\tiny #1}\typeout{Notes!
Notes!
Notes!
}}
\renewcommand{\notes}[1]{\typeout{notes!}}

\newcommand{\field}[1]{\mathbb{#1}}

\def\Re{\field{R}}

\def\Sec#1{Sec.~\ref{#1}}

\newcommand{\keps}{{k}_{\epsilon}}

\def\exp{\text{exp}}

\makeatletter

\newcommand{\Rom}[1]{\expandafter\@slowromancap\romannumeral #1@}
\makeatother

\newcounter{rmnum}

\newenvironment{romannum}{\begin{list}{{\upshape (\roman{rmnum})}}{\usecounter{rmnum}
\setlength{\leftmargin}{17pt}
\setlength{\rightmargin}{4pt}
\setlength{\itemindent}{-1pt}
}}{\end{list}}

\newcounter{anum}

\newcommand{\ud}{\,\mathrm{d}}

\newcommand{\pr}{\rho}

\def\Expect{{\sf E}}

\newcommand{\lr}[2]{\langle #1, #2 \rangle}

\newcommand{\newP}[1]{\noindent{\bf #1:}}

\def\Expect{{\sf E}}

\usepackage{ulem}

\newtheorem{theorem}{Theorem}

\newtheorem{proposition}{Proposition}

\newtheorem{remark}{Remark}

\makeatletter
\newsavebox\myboxA
\newsavebox\myboxB
\newlength\mylenA

\newcommand*\xoverline[2][0.75]{%
    \sbox{\myboxA}{$\m@th#2$}%
    \setbox\myboxB\null
    \ht\myboxB=\ht\myboxA%
    \dp\myboxB=\dp\myboxA%
    \wd\myboxB=#1\wd\myboxA
    \sbox\myboxB{$\m@th\overline{\copy\myboxB}$}
    \setlength\mylenA{\the\wd\myboxA}
    \addtolength\mylenA{-\the\wd\myboxB}%
    \ifdim\wd\myboxB<\wd\myboxA%
       \rlap{\hskip 0.5\mylenA\usebox\myboxB}{\usebox\myboxA}%
    \else
        \hskip -0.5\mylenA\rlap{\usebox\myboxA}{\hskip 0.5\mylenA\usebox\myboxB}%
    \fi}
\makeatother

\def\argmax{\mathop{\text{\rm arg\,max}}}

\newcommand{\Pspace}{{\cal P}_{\text{ac},2}}
\newcommand{\LX}{{\sf L}}
\newcommand{\LP}{{\cal L}}
\newcommand{\HP}{{\cal H}}
\newcommand{\HX}{{\sf H}}
\newcommand{\JX}{{\sf J}}

\newcommand{\fSpace}{\mathcal{C}(\Re^d;\Re)}

\newcommand{\geps}{g_\epsilon}
\newcommand{\FP}{{\sf F}} 

\title{Accelerated Flow for Probability Distributions}

\author{
	Amirhossein Taghvaei \& Prashant G. Mehta \thanks{Coordinated Science Laboratory, University of Illinois at Urbana-Champaign}	
}

\begin{document}
	
	\maketitle
	
	\begin{abstract}
		This paper presents a methodology and numerical algorithms for constructing accelerated gradient flows on the space of probability distributions.  In
		particular, we 
		extend the recent variational formulation of
		accelerated gradient methods in~\cite{wibisono2016}
		from vector valued variables to probability
		distributions.  The variational problem
		is modeled as a mean-field optimal control problem.
		The maximum principle of optimal control theory is used to derive Hamilton's
		equations for the optimal gradient flow. 
		The Hamilton's equation are shown to achieve the
		accelerated form of density
		transport from any initial
		probability distribution to a target probability distribution.  
		A quantitative estimate on the asymptotic convergence rate is provided
		based on a Lyapunov function construction, when the
		objective functional is displacement convex.  Two
		numerical approximations are presented to implement the
		Hamilton's equations as a system of $N$ interacting
		particles.  The continuous limit of the Nesterov's
		algorithm is shown to be a special case with $N=1$. The algorithm is illustrated with
		numerical 
		examples.    
	\end{abstract}
	
	\section{Introduction}
	
	Optimization on the space of probability distributions is important to
	a number of machine learning models including variational
	inference~\cite{blei2017variational}, generative
	models~\cite{goodfellow2014generative,arjovsky2017wasserstein}, and
	policy optimization in reinforcement learning~\cite{sutton2000policy}.
	A number of recent studies have considered solution approaches to
	these problems based upon a construction of gradient flow on the space
	of probability
	distributions~\cite{zhang2018policy,liu2016stein,frogner2018approximate,chizat2018global,richemond2017wasserstein,chen2018unified}.
	Such constructions are useful for convergence analysis as
	well as development of numerical algorithms.
	
	In this paper, we propose a methodology and numerical algorithms that
	achieve {\em accelerated} gradient flows on the space of probability
	distributions.  The proposed numerical algorithms are related to yet
	distinct from the accelerated
	stochastic gradient descent~\cite{jain2017accelerating} and
	Hamiltonian Markov chain Monte-Carlo (MCMC)
	algorithms~\cite{neal2011mcmc,cheng2017underdamped}. 
	The proposed methodology extends the variational formulation of~\cite{wibisono2016}
	from vector valued variables to probability distributions.  The
	original formulation of~\cite{wibisono2016} was used to derive and
	analyze the convergence properties of a
	large class of accelerated optimization algorithms, most significant
	of which is the continuous-time limit of the Nesterov's
	algorithm~\cite{su2014}.  In this paper, the limit is referred to as
	the Nesterov's ordinary differential equation (ODE).  
	
	The extension proposed in our work is based upon a generalization of
	the formula for the Lagrangian in~\cite{wibisono2016}: (i) the kinetic energy
	term is replaced with the expected value of kinetic energy; and (ii)
	the potential energy term is
	replaced with a suitably defined functional on the space of
	probability distributions.  
	The variational problem is to obtain a trajectory in
	the space of probability distributions that minimizes the action
	integral of the Lagrangian.

	
	The variational problem is modeled as a mean-field optimal problem.
	The maximum principle of the optimal control theory is
	used to derive the Hamilton's equations which represent the first
	order optimality conditions.  The Hamilton's equations provide a
	generalization of the Nesterov's ODE to the space of probability
	distributions.  A candidate Lyapunov function is proposed for the
	convergence analysis of the solution of the Hamilton's equations.  In
	this way, quantitative estimates on convergence rate are obtained for
	the case when the objective functional is displacement
	convex~\cite{mccann1997convexity}. Table~\ref{tab:summary} provides a
	summary of the relationship between the original variational formulation
	in~\cite{wibisono2016} and the extension proposed in this paper.

	We also consider the important special case when the objective
	functional is the relative entropy functional ${\sf
		D}(\rho|\rho_\infty)$ defined with respect to a
	target probability distribution $\rho_\infty$.  In this case, the
	accelerated gradient flow is shown to be 
	related to the continuous limit of the Hamiltonian Monte-Carlo
	algorithm~\cite{cheng2017underdamped}
	(Remark~\ref{rem:relative-entropy}).  
	The Hamilton's equations are finite-dimensional for the
	special case when the initial and the 
	target probability distributions are both Gaussian. In this case, the mean evolves according to the Nesterov's
	ODE.  For the general case, the Lyapunov function-based convergence
	analysis applies when the target distribution is log-concave.

	
	As a final contribution, the proposed methodology is used to obtain a
	numerical algorithm.  The algorithm is an interacting
	particle system that empirically approximates the distribution with a 
	finite but large number of $N$ particles.  The
	difficult part of this construction is the approximation of the
	interaction term between particles.  For this purpose, two
	types of approximations are described: (i) Gaussian approximation
	which is asymptotically (as $N\rightarrow\infty$) exact 
	in Gaussian settings; and (ii) Diffusion map approximation which is
	computationally more demanding but asymptotically exact for a 
	general class of distributions.

	The outline of the remainder of this paper is as follows:
	\Sec{sec:review}
	provides a brief review of the variational formulation
	in~\cite{wibisono2016}. The proposed extension to the space of
	probability distribution appears in~\Sec{sec:main}
	where the main result is also described.  The numerical algorithm along with the results of numerical experiments
	appear in~\Sec{sec:numerical}. Comparisons with MCMC and Hamiltonian MCMC are also described. The conclusions appear in \Sec{sec:conclusion}. 
	
	\newP{Notation} The gradient and divergence operators are
	denoted as $\nabla$ and $\nabla \cdot$ respectively. With
	multiple variables, $\nabla_z$ denotes the
	gradient with respect to the variable $z$.  Therefore, the divergence of the vector field $U$ is $\nabla \cdot U(x)= \sum_{n=1}^d \nabla_{x_n}U_n(x)$. 
	The space of absolutely continuous probability measures on $\Re^d$ with finite second moments is denoted by $\Pspace(\Re^d)$. 
	The Wasserstein gradient 
	and the Gâteaux derivative 
	of a functional ${\sf F}$ is denoted as
	$\nabla_\rho {\sf F}(\rho)$
	and $\frac{\partial {\sf F}}{\partial \rho} (\rho)$ respectively~(see Appendix~\ref{apdx:derivative} for definition). 
	The probability distribution of a random variable $Z$ is denoted as $\text{Law}(Z)$.
	
	{\renewcommand{\arraystretch}{1.1}
	\begin{table}[h]
		\centering
		\begin{tabular}{l|c|c}
			&  Vector & Probability distribution \\ \hline
			State-space &  $\Re^d$ & $\mathcal{P}_2(\Re^d)$\\  
			Objective function  &  $f(x)$ & ${\sf F}(\pr) := {\sf D}(\pr\vert\pr_\infty)$\\    
			Lagrangian &
			$e^{\alpha_t+\gamma_t}\left(\frac{1}{2}|e^{-\alpha_t}u|^2
			- e^{\beta_t}f(x)\right)$ &
			$e^{\alpha_t+\gamma_t}\Expect\left[\frac{1}{2}|e^{-\alpha_t}U|^2
			- e^{\beta_t}\log(\frac{\pr(X)}{\pr_\infty(X)})\right]$
			\\
			\multirow{2}{*}{Lyapunov funct.} & $
			\frac{1}{2}\left|x+e^{-\gamma_t}y
			- \bar{x}\right|^2 $
			&
			$\frac{1}{2}\Expect [|X_t + e^{-\gamma_t}Y_t - T_{\pr_t}^{\pr_\infty}(X_t)|^2] $
			\\
			& $ + e^{\beta_t}(f(x)-f(\bar{x}))$
			&
			$ + e^{\beta_t}({\sf F}(\pr_t)-{\sf F}(\pr_\infty))$
		\end{tabular}
		\caption{Summary of the variational formulations for vectors and probability distributions. }
		\label{tab:summary}
	\end{table}
	{\renewcommand{\arraystretch}{1}

	\section{Review of the variational formulation of~\cite{wibisono2016}}\label{sec:review}
	
	The basic problem is to minimize a $C^1$ smooth convex function $f$ on $\Re^d$.
	The standard form of the gradient descent algorithm for this problem
	is an ODE:
	\begin{equation}
	\frac{\ud X_t}{\ud t} = -\nabla f(X_t),\quad t \geq 0\label{eq:grad-flow}
	\end{equation}

	Accelerated forms of this algorithm are obtained based on a
	variational formulation due to~\cite{wibisono2016}.  The formulation
	is briefly reviewed here using an optimal control formalism. The
	Lagrangian $L:\Re^+\times \Re^d  \times \Re^d \to \Re$   is defined as
	\begin{equation}
	L(t,x,u) := e^{\alpha_t+\gamma_t}\bigg(\underbrace{\frac{1}{2}|e^{-\alpha_t}u|^2}_{\text{kinetic energy}} - \underbrace{e^{\beta_t}f(x)}_{\text{potential energy}}\bigg)
	\label{eq:Lagrangian}
	\end{equation}
	where $t \geq 0$ is the time, $x \in \Re^d$ is the state, $u \in \Re^d$ is the velocity or control input, and the time-varying parameters $\alpha_t,\beta_t,\gamma_t$ satisfy the following scaling conditions: $\alpha_t = \log p - \log t$, $\beta_t = p\log t + \log C$, and $\gamma_t = p \log t$ where $p\geq 2$ and $C>0$ are constants. 
	
	The variational problem is
	\begin{equation}
	\begin{aligned}
	\underset{u}{\text{Minimize}}\quad&J(u)=\int_{0}^\infty L(t,X_t,u_t) \ud t \\
	\text{Subject to}\quad& \frac{\ud X_t}{\ud t} = u_t,\quad X_0 =x_0
	\end{aligned}
	\label{eq:var-problem}
	\end{equation}
	The Hamiltonian function is
	\begin{equation}
	H(t,x,y,u) = y  \cdot u  - L(t,x,u)
	\end{equation}
	where $y \in \Re^d$ is dual variable and $y\cdot u$ denotes the dot
	product between vectors $y$ and $u$.
	
	According to the Pontryagin's Maximum Principle, the optimal control
	$u^*_t = \underset{v}{\argmax}~ H(t,X_t,Y_t,v) = e^{\alpha_t-\gamma_t}Y_t$.  The
	resulting Hamilton's equations are
	\begin{subequations}
		\begin{align}
		\frac{\ud X_t}{\ud t} &=+ \nabla_y H(t,X_t,Y_t,u_t)= e^{\alpha_t-\gamma_t}Y_t,\quad X_0=x_0 \\
		\frac{\ud Y_t}{\ud t} &=-\nabla_x H(t,X_t,Y_t,u_t)= - e^{\alpha_t+\beta_t+\gamma_t}\nabla f(X_t),\quad Y_0=y_0 
		\end{align}
		\label{eq:acc-grad-flow}	
	\end{subequations}
	The system~\eqref{eq:acc-grad-flow} is an example of accelerated
	gradient descent algorithm.  Specifically, if the parameters
	$\alpha_t,\beta_t,\gamma_t$ are defined using $p=2$, one obtains the
	continuous-time limit of the Nesterov's accelerated algorithm.  It is
	referred to as the Nesterov's ODE in this paper.

	For this system, a Lyapunov function is as follows:
	\begin{equation}
	V(t,x,y)= \frac{1}{2}\left|x+e^{-\gamma_t}y - \bar{x}\right|^2 + e^{\beta_t}(f(x)-f(\bar{x}))\label{eq:Lyp}
	\end{equation}
	where $\bar{x} \in \argmin_x f(x)$. It is shown in~\cite{wibisono2016}
	that upon differentiating along the solution trajectory, 
	$
	\frac{\ud}{\ud t}V(t,X_t,Y_t) \leq 0
	$.  This yields the following convergence rate:
	\begin{equation}
	f(X_t)- f(\bar{x}) \leq O(e^{-\beta_t}),\quad \forall t\geq 0
	\label{eq:conv-det}
	\end{equation}

	\section{Variational formulation for probability
		distributions} \label{sec:main}
	
	\subsection{Motivation and background} 
	Let ${\sf F}:\Pspace(\Re^d)\to \Re$ be a functional on the space of probability distributions. Consider the problem of  minimizing ${\sf F}(\pr)$.
	The (Wasserstein) gradient flow with respect to ${\sf F}(\pr)$ is
	\begin{equation}\label{eq:grad-flow-pde}
	\frac{\partial \pr_t}{\partial t} = \nabla \cdot(\pr_t \nabla_\pr {\sf F}(\pr_t))
	\end{equation}
	where $\nabla_\pr {\sf F}(\pr)$ is the Wasserstein gradient of ${\sf F}$. 
	
	An important example is the relative entropy
	functional where ${\sf F}(\pr)= {\sf D}(\pr \vert
	\pr_\infty):=\int_{\Re^d} \log(\frac{\pr(x)}{\pr_\infty(x)})\pr(x)\ud
	x$ where $\pr_\infty \in \Pspace(\Re^d)$ is referred to as the target distribution.  The
	gradient of relative entropy is given by $\nabla_\pr {\sf F}(\pr) = \nabla
	\log(\frac{\pr}{\pr_\infty})$. The gradient flow 
	\begin{equation}\label{eq:Fokker-Plank}
	\frac{\partial
		\pr_t}{\partial t} = -\nabla \cdot(\pr_t\nabla \log(\pr_\infty)) +
	\Delta \pr_t
	\end{equation} 
	is the Fokker-Planck
	equation~\cite{jordan1998variational}.
	The gradient flow achieves the density
	transport from an initial probability distribution $\pr_0$ to the
	target (here, also equilibrium) probability distribution $\pr_\infty$; and underlies the
	construction and the analysis of Markov
	chain Monte-Carlo (MCMC) algorithms.  The simplest MCMC algorithm is the Langevin stochastic differential equation (SDE):
	\begin{equation*}
	\ud X_t = -\nabla f(X_t)\ud t + \sqrt{2}\ud B_t,\quad X_0\sim\pr_0
	\end{equation*}
	where $B_t$ is the standard Brownian motion in $\Re^d$. 
	
	
	
	The main problem of this paper is to construct an accelerated form of
	the gradient flow~\eqref{eq:grad-flow-pde}.  
	The proposed
	solution is based upon a variational formulation.  As tabulated in
	Table~\ref{tab:summary}, the solution represents a generalization
	of~\cite{wibisono2016} from its original deterministic
	finite-dimensional to now probabilistic infinite-dimensional settings. 
	
	The variational problem can be expressed in two equivalent forms: (i)
	The probabilistic form is described next in the main body of the
	paper; and (ii) The partial differential equation (PDE) form appears in the
	Appendix.  The probabilistic form is stressed here because it represents
	a direct generalization of the Nesterov's ODE and because it is closer
	to the numerical algorithm.

	\subsection{Probabilistic form of the variational problem}\label{sec:prob}
	Consider the stochastic process $\{X_t\}_{t\geq 0}$
	that takes values in $\Re^d$ and evolves according to:
	\begin{equation*}
	\frac{\ud X_t}{\ud t} = U_t,\quad X_0 \sim \rho_0
	\end{equation*}
	where the control input $\{U_t\}_{t\geq 0}$
	also takes values in $\Re^d$, and $\rho_0\in\Pspace(\Re^d)$ is the probability
	distribution of the initial condition $X_0$.  It is noted that the
	randomness here comes only from the random initial condition.

	Suppose the objective functional is of the form ${\sf F}(\pr)=\int \tilde{F}(\pr,x)\pr(x)\ud x$. The Lagrangian $\LX:\Re^+\times \Re^d \times\Pspace(\Re^d) \times \Re^d \to \Re$ is  defined as
	\begin{equation}
	\LX(t,x,\rho,u) :=
	e^{\alpha_t+\gamma_t}\bigg(\underbrace{\frac{1}{2}|e^{-\alpha_t}u|^2}_{\text{kinetic
			energy}} -\underbrace{e^{\beta_t} \tilde{F}(\pr,x)}_{\text{potential energy}}  \bigg)
	\end{equation}
	This formula is a natural generalization of the Lagrangian~\eqref{eq:Lagrangian} and the parameters $\alpha_t,\beta_t,\gamma_t$ are defined exactly the same
	as in the finite-dimensional case. 
	The stochastic optimal control problem is:
	\begin{equation}
	\begin{aligned}
	\text{Minimize}\quad& \JX(u)=\Expect\left[\int_0^\infty \LX(t,X_t,\pr_t,U_t)\ud t\right] \\
	\text{Subject to}\quad& \frac{\ud X_t}{\ud t} = U_t,\quad X_0 \sim \rho_0
	\end{aligned}
	\label{eq:var-problem-X}
	\end{equation}
	where $\pr_t=\text{Law}(X_t)\in \Pspace(\Re^d)$ is the probability density function of the random variable $X_t$. 
	
	The Hamiltonian function $\HX:\Re^+\times \Re^d
	\times\Pspace(\Re^d) \times \Re^d \times \Re^d \to \Re$ for
	this problem is given by~\cite[Sec. 6.2.3]{carmona2017probabilistic}:
	\begin{equation}\label{eq:Hamilton-function-X}
	\HX(t,x,\rho,y,u):= u \cdot y- {\sf L}(t,x,\rho,u)
	\end{equation}
	where  $y \in \Re^d$ is the dual variable.
	
%
	
	\subsection{Main result}
	\begin{theorem}\label{thm:main-res}
		Consider the variational problem
		\eqref{eq:var-problem-X}.  
		\begin{romannum}
			\item The optimal control $U_t^* =
			e^{\alpha_t-\gamma_t}Y_t$ where the optimal trajectory $\{(X_t,Y_t)\}_{t \geq 0}$
			evolves according to the Hamilton's equations:
			\begin{subequations}				\label{eq:Hamilton-X}
				\begin{align}
				\frac{\ud X_t }{\ud t} &= U_t^* =
				e^{\alpha_t-\gamma_t}Y_t,\quad
				X_0 \sim \rho_0\\
				\frac{\ud Y_t }{\ud t} &=  - e^{\alpha_t+\beta_t+\gamma_t}\nabla_\pr {\sf F}(\pr_t)(X_t),\quad Y_0 = \nabla \phi_0(X_0)
				\end{align}
			\end{subequations}
			where $\phi_0$ is any convex function and $\pr_t:=\text{Law}(X_t)$. 
			\vspace{-10pt}
			\item Suppose also that the functional ${\sf F}$ is
			displacement convex and $\pr_\infty$ is its
			minimizer.  Define the energy along the
			optimal trajectory 
			\begin{equation}\label{eq:Lyapunov-function}
			V(t) = \frac{1}{2}\Expect [|X_t + e^{-\gamma_t}Y_t - T_{\pr_t}^{\pr_\infty}(X_t)|^2] + e^{\beta_t}({\sf F}(\pr)-{\sf F}(\pr_\infty))
			\end{equation}
			where the map $T_{\pr_t}^{\pr_\infty}:\Re^d\to \Re^d$ is the optimal transport map from $\pr_t$ to $\pr_\infty$.  Suppose also that the following technical assumption holds: $\Expect[(X_t+e^{-\gamma_t}Y_t -
			T^{\pr_\infty}_{\pr_t}(X_t))\cdot \frac{\ud}{\ud t} T^{\pr_\infty}_{\pr_t}(X_t)]=0$. Then $\frac{\ud V}{\ud t}(t) \leq 0$. Consequently, the following rate of convergence  is obtained along the optimal trajectory:
			\begin{equation}\label{eq:conv-stoch}
			{\sf F}(\pr_t)-{\sf F}(\pr_\infty) \leq O(e^{-\beta_t}),\quad \forall  t\geq 0
			\end{equation}
		\end{romannum}
		%
	\end{theorem}
	\begin{proof}[Proof sketch]
		The Hamilton's equations are derived using the standard mean-field optimal control theory~\cite{carmona2017probabilistic}. The Lyapunov function argument is based upon the variational inequality characterization of a displacement convex function~\cite[Eq. 10.1.7]{ambrosio2008gradient}.
		The detailed proof appears in the Appendix. We expect that the technical assumption is not necessary. This is the subject of the continuing work. 
	\end{proof}	
	

	
	\subsection{Relative entropy as the functional}\label{sec:relative-entropy}
	
	In the remainder of this paper, we assume that the functional
	${\sf F}(\pr)={\sf D}(\pr|\pr_\infty)$ is the relative entropy where
	$\pr_\infty \in \Pspace(\Re^d)$ is a given target probability distribution.
	In this case the Hamilton's equations are given by
	\begin{subequations}		\label{eq:Hamilton-relative-entropy}
		\begin{align}
		\frac{\ud X_t}{\ud t}  & =
		e^{\alpha_t-\gamma_t}Y_t,\quad
		X_0 \sim \rho_0\\
		\frac{\ud  Y_t}{\ud t} &=  - e^{\alpha_t+\beta_t+\gamma_t}(\nabla f(X_t)  + \nabla \log(\pr_t(X_t)),\quad Y_0=\nabla \phi_0(X_0)
		\end{align}
	\end{subequations}
	where $\pr_t=\text{Law}(X_t)$ and $f=-\log(\pr_\infty)$. Moreover, if $f$ is
	convex (or equivalently $\pr_\infty$ is log-concave), then $\FP$ is
	displacement convex with the unique minimizer at $\pr_\infty$ and the
	convergence estimate is given by ${\sf D}(\pr_t\vert\pr_\infty)\leq O(e^{-\beta_t})$. 
	\begin{remark}\label{rem:relative-entropy}
		The Hamilton's equations~\eqref{eq:Hamilton-relative-entropy}
		with the relative entropy functional 
		is related to the under-damped Langevin
		equation~\cite{cheng2017underdamped}. The difference is that
		the deterministic term $\nabla \log(\pr_t)$
		in~\eqref{eq:Hamilton-relative-entropy} is replaced with a
		random 
		Brownian motion term in the under-damped Langevin
		equation. More detailed comparison appears in the Appendix~\ref{apdx:Langevin}.
	\end{remark}
	\subsection{Quadratic Gaussian case}\label{sec:Gaussian}
	Suppose the initial distribution $\rho_0$ and the target distribution
	$\rho_\infty$ are both Gaussian, denoted as ${\cal N}(m_0,\Sigma_0)$
	and ${\cal N}(\bar{x},Q)$, respectively. 
	This is equivalent to the  objective function $f(x)$ being quadratic
	of the form $f(x) = \frac{1}{2}(x-\bar{x})^\top Q^{-1}(x-\bar{x})$. 
	Therefore, this problem is referred to as the {\it quadratic Gaussian
		case}. The following Proposition shows that the mean of the
	stochastic process $(X_t,Y_t)$ evolves according to the
	Nesterov ODE~\eqref{eq:acc-grad-flow}:
	\begin{proposition}(Quadratic Gaussian case) \label{prop:Gaussian}
		Consider the variational problem~\eqref{eq:var-problem-X} for
		the quadratic Gaussian case.  
		Then 
		\begin{romannum}
			\item The stochastic process $(X_t,Y_t)$ is a Gaussian process. The Hamilton's equations
			are given by:
			\begin{align*}
			\frac{\ud X_t}{\ud t}  &= e^{\alpha_t-\gamma_t}Y_t,\quad
			\frac{\ud Y_t}{\ud t}  =  - e^{\alpha_t+\beta_t+\gamma_t} (Q^{-1}(X_t-\bar{x}) -\Sigma_t^{-1}(X_t-m_t))
			\end{align*}    
			where $m_t$ and $\Sigma_t$ are the mean and
			the covariance of $X_t$. 
			\item 
			Upon taking the expectation of both sides, and denoting $n_t:=\Expect[Y_t]$
			\begin{align*}
			\frac{\ud m_t }{\ud t} &= e^{\alpha_t-\gamma_t}n_t,\quad
			\frac{\ud  n_t}{\ud t} =  - e^{\alpha_t+\beta_t+\gamma_t} \underbrace{Q^{-1}(m_t-\bar{x})}_{\nabla f(m_t)}
			\end{align*}
			which is identical to Nesterov ODE~\eqref{eq:acc-grad-flow}.
		\end{romannum}
	\end{proposition}

	\section{Numerical algorithm} \label{sec:numerical}
	The proposed numerical algorithm is based upon an interacting particle
	implementation of the Hamilton's
	equation~\eqref{eq:Hamilton-relative-entropy}.  
	Consider a system of $N$ particles $\{(X^i_t,Y^i_t)\}_{i=1}^N$
	that evolve according to:
	\begin{subequations}
		\begin{align*}
		\frac{\ud X^i_t}{\ud t}  &= e^{\alpha_t-\gamma_t}Y^i_t ,\quad X^i_0 \sim \rho_0\\
		\frac{\ud Y^i_t }{\ud t} &=  - e^{\alpha_t+\beta_t+\gamma_t} (\nabla f(X^i_t) + \underbrace{I^{(N)}_t(X^i_t) }_{\text{interaction term}}),\quad Y^i_0 =\nabla \phi_0(X^i_0)
		\end{align*}
		\label{eq:Hamilton-flow-p-i}
	\end{subequations} 
	The interaction term $I^{(N)}_t$ is an empirical approximation of the $\nabla
	\log(\rho_t) $ term in~\eqref{eq:Hamilton-relative-entropy}.  
	We propose two types of empirical approximations as follows:
	
	\textbf{1. Gaussian approximation:}  Suppose the
	density is approximated as a
	Gaussian~$\mathcal{N}(m_t,\Sigma_t)$. In this case, $\nabla \log(\rho_t(x)) = - {\Sigma_t}^{-1} (x-m_t) $.
	This motivates the following empirical approximation of
	the interaction term:
	\begin{align}\label{eq:interaction-Gaussian}
	I^{(N)}_t(x) =- {\Sigma_t^{(N)}}^{-1} (x-m^{(N)}_t) 
	\end{align} 
	where $m_t^{(N)}:=N^{-1} \sum_{i=1}^N X^i_t$ is
	the empirical mean and
	$\Sigma_t^{(N)}:=\frac{1}{N-1}\sum_{i=1}^N
	(X^i_t-m^{(N)}_t)(X^i_t-m^{(N)}_t)^\top$ is the
	empirical covariance. 
	
	Even though the approximation is asymptotically (as $N\rightarrow\infty$) exact only under the Gaussian assumption, it may
	be used in a more general settings, particularly
	when the density $\rho_t$ is unimodal. The situation is analogous to
	the (Bayesian) filtering problem, where an ensemble Kalman filter is
	used as an approximate solution for non-Gaussian distributions~\cite{evensen2003ensemble}.
	
	\textbf{2. Diffusion map approximation:} This
	is based upon
	the diffusion map approximation of the weighted
	Laplacian
	operator~\cite{coifman2006diffusion,hein2007graph}.
	For a $C^2$ function $f$, the weighted Laplacian is defined as $\Delta_\pr f :=                  \frac{1}{\pr}\nabla \cdot(\pr \nabla f)$.  Denote $e(x)=x$ as the
	coordinate function on $\Re^d$.  It is a
	straightforward calculation to show that $\nabla
	\log(\pr) = \Delta_\pr e$.  This allows one to use
	the diffusion map approximation of the weighted Laplacian
	to approximate the interaction term as follows:  
	\begin{equation}\label{eq:interaction-kernel}
	\text{(DM)}\quad I_t^{(N)}(X^i_t)=\frac{1}{\epsilon}\frac{\sum_{j=1}^N \keps(X^i_t,X^j_t)(X^j_t-X^i_t)}{\sum_{j=1}^N \keps(X^i_t,X^j_t)}
	\end{equation}
	where the kernel 
	$\keps(x,y) = \frac{\geps(x,y)}{\sqrt{\sum_{i=1}^N \geps(y,X^i)}}$
	is constructed empirically in terms of the Gaussian
	kernel $\geps(x,y)=\exp(-|x-y|^2/(4\epsilon))$.  The
	 parameter $\epsilon$ is referred to as the kernel bandwidth. 
	The approximation is asymptotically exact as $\epsilon \downarrow 0$ and $N \uparrow \infty$.
	The approximation error is of order $O(\epsilon) + O(\frac{1}{\sqrt{N}\epsilon^{d/4}})$ where the first term is referred to as the bias error and the second term is referred to as the variance error~\cite{hein2007graph}. The variance error is the dominant term in the error for small values of $\epsilon$, whereas the bias error is the dominant term for large values of $\epsilon$ (see Figure~\ref{fig:comparison-eps}).
	
	The resulting interacting particle algorithm is tabulated in
	Table~\ref{alg:Gaussian-kernel}.  The symplectic method proposed
	in~\cite{betancourt2018symplectic} is used to carry out the numerical
	integration.  The algorithm is applied to two examples as described in
	the following sections.

	\begin{algorithm}[t]
		\caption{Interacting particle implementation of the accelerated gradient flow}
		\begin{algorithmic}
			\REQUIRE $\pr_0$, $\phi_0$, $N$, $t_0$, $\Delta t$, $p$, $C$, $K$
			\ENSURE $\{X^i_k\}_{i=1,k=0}^{N,K}$
			\STATE Initialize $\{X^i_0\}_{i=1}^N \overset{\text{i.i.d}}{\sim} \pr_0$, $Y^i_0=\nabla \phi_0(X^i_0)$  
			\STATE Compute $I^{(N)}_0(X^i_0)$ with \eqref{eq:interaction-Gaussian} or~\eqref{eq:interaction-kernel}
			\FOR   {$k=0$ to  $K-1$}
			\STATE $t_{k+\frac{1}{2}}=t_k + \frac{1}{2}\Delta t$
			\STATE $Y^i_{k+\frac{1}{2}} =Y^i_{k} - \frac{1}{2}Cpt_{k+\frac{1}{2}}^{2p-1}(\nabla f(X^i_k) + I^{(N)}_k(X^i_k))\Delta t$ 
			\STATE $X^i_{k+1} = X^i_{k} + \frac{p}{t_{k+\frac{1}{2}}^{p+1}}Y^i_k\Delta t$ 
			\STATE Compute $I^{(N)}_{k+1}(X^i_{k+1})$ with~\eqref{eq:interaction-Gaussian}or~\eqref{eq:interaction-kernel}
			\STATE $Y^i_{k+1} =Y^i_{k+\frac{1}{2}} - \frac{1}{2}Cpt_{k+\frac{1}{2}}^{2p-1}(\nabla f(X^i_{k+1}) + I^{(N)}_{k+1}(X^i_{k+1}))\Delta t$ 
			\STATE $t_{k+1}=t_{k+\frac{1}{2}} + \frac{1}{2}\Delta t$	
			\ENDFOR		
		\end{algorithmic}
		\label{alg:Gaussian-kernel}
	\end{algorithm}
	
	\begin{remark}
		For the case where there is only one particle (
		$N=1$), the interaction term is zero and the
		system~\eqref{eq:Hamilton-flow-p-i} reduces to the
		Nesterov ODE~\eqref{eq:acc-grad-flow}.  
	\end{remark}
	
	\begin{remark}(Comparison with density estimation)
		The diffusion map approximation algorithm is conceptually different from an explicit density estimation-based approach. A basic density estimation is to approximate $\pr(x) \approx \frac{1}{N}\sum_{i=1}^N \geps(x,X^i_t)$ where $\geps(x,y)$ is the Gaussian kernel. Using such an approximation, the interaction term is approximated as 
		\begin{equation}\label{eq:density-estimation}
		\text{(DE)}\quad I^{(N)}_t(X^i_t) = \frac{1}{2\epsilon}\frac{\sum_{j=1}^N \geps(X^i_t,X^j_t)(X^j_t-X^i_t)}{\sum_{j=1}^N \geps(X^i_t,X^j_t)}
		\end{equation} 
		Despite the apparent similarity of the two formulae,~\eqref{eq:interaction-kernel} for diffusion map approximation and~\eqref{eq:density-estimation} for density estimation, the nature of the two approximations is different. The difference arises because the kernel $\keps(x,y)$ in~\eqref{eq:interaction-kernel} is data-dependent whereas the kernel in~\eqref{eq:density-estimation} is not.
		While both approximations are exact in the asymptotic limit as $N\uparrow\infty$ and $\epsilon \downarrow 0$, they exhibit different convergence rates. Numerical experiments presented in  Figure~\ref{fig:comparison-N}-(d) show that the diffusion map approximation has a much smaller variance for intermediate values of $N$. Theoretical understanding of the difference is the subject of continuing  work.  
	\end{remark}
	
	\subsection{Gaussian Example}\label{sec:num-Gaussian}
	Consider the Gaussian example as described in Sec.~\ref{sec:Gaussian}. 
	The simulation results for the scalar ($d=1$) case
	with initial distribution $\pr_0=\mathcal{N}(2,4)$ and
	target distribution $\mathcal{N}(\bar{x},Q)$ where
	$\bar{x}=-5.0$ and $Q=0.25$ is depicted in
	Figure~\ref{fig:Gaussian-res}-(a)-(b).  For this
	simulation, the numerical
	parameters are as follows: $N=100$,
	$\phi_0(x)=0.5(x-2)$, $t_0=1$, $\Delta t= 0.1$, $p=2$,$C=0.625$, and $K=400$. The result numerically  verifies the $O(e^{-\beta_t})=O(\frac{1}{t^2})$ convergence rate derived in Theorem~\ref{thm:main-res}  for the case where the target distribution is Gaussian.

	\begin{figure}[t]
		\begin{tabular}{cc}
			\subfigure[]{
				\includegraphics[width=0.5\columnwidth]{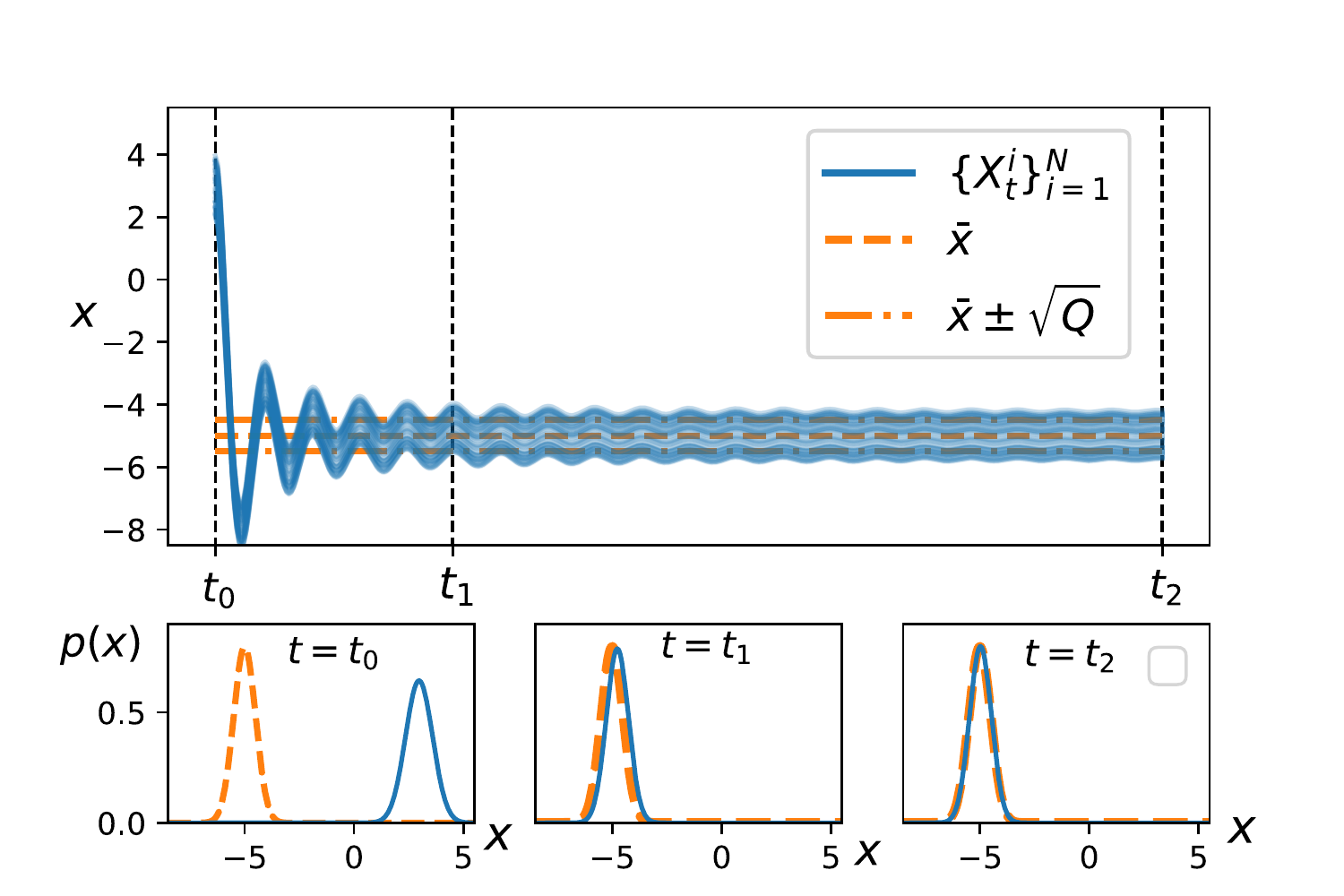}
				
			}&
			\subfigure[]{
				\includegraphics[width=0.5\columnwidth]{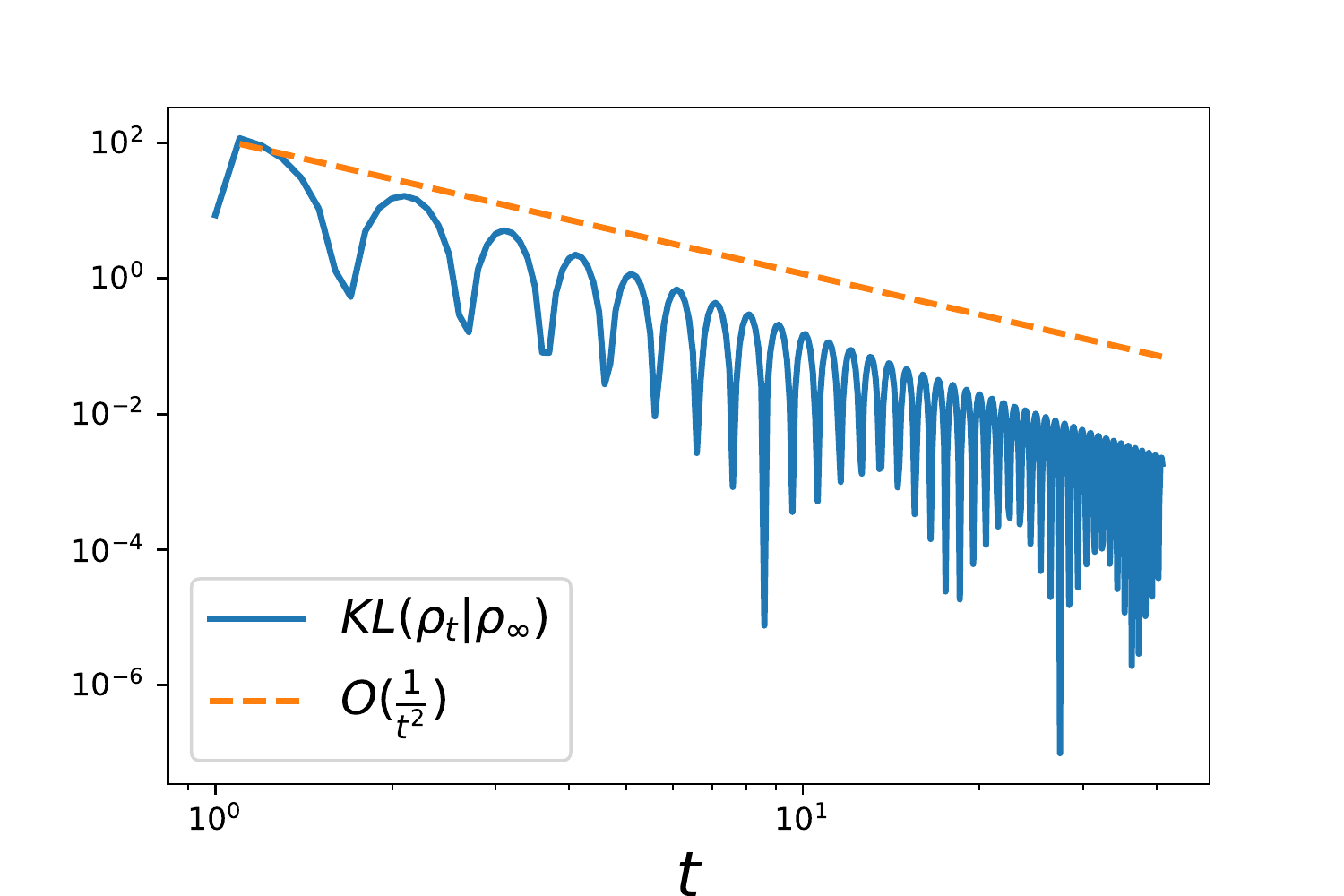}
				
			}	
		\end{tabular}
		\caption{Simulation result for the Gaussian case
			(Example~\ref{sec:num-Gaussian}): (a) The time
			traces of the particles; (b) The KL-divergence as a
			function of time.   
			}
		\label{fig:Gaussian-res}
	\end{figure}
	\begin{figure}[t]
	\begin{tabular}{cc}
		\subfigure[]{
			\includegraphics[width=0.5\columnwidth]{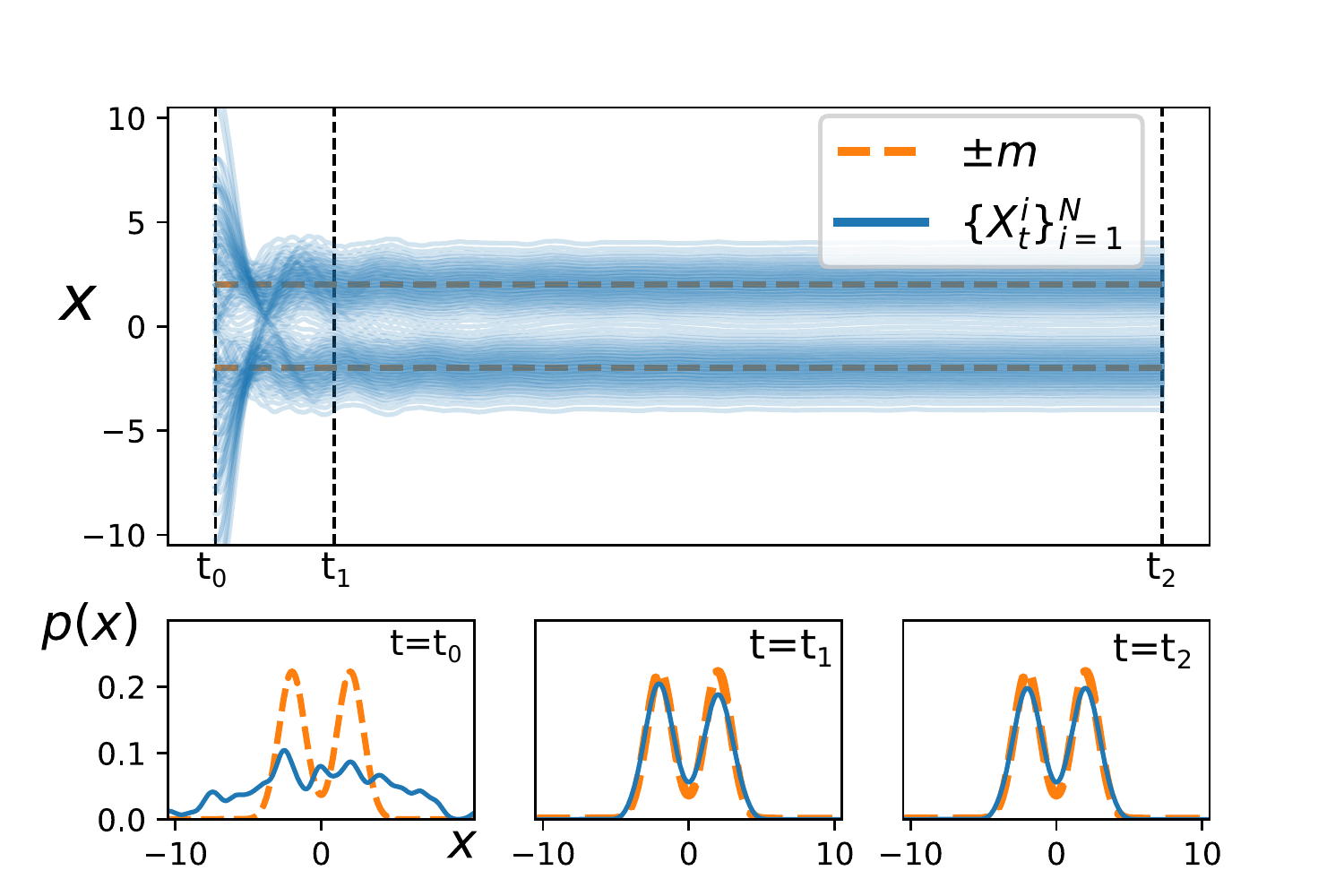}
			
		}&
		\subfigure[]{
			\includegraphics[width=0.5\columnwidth]{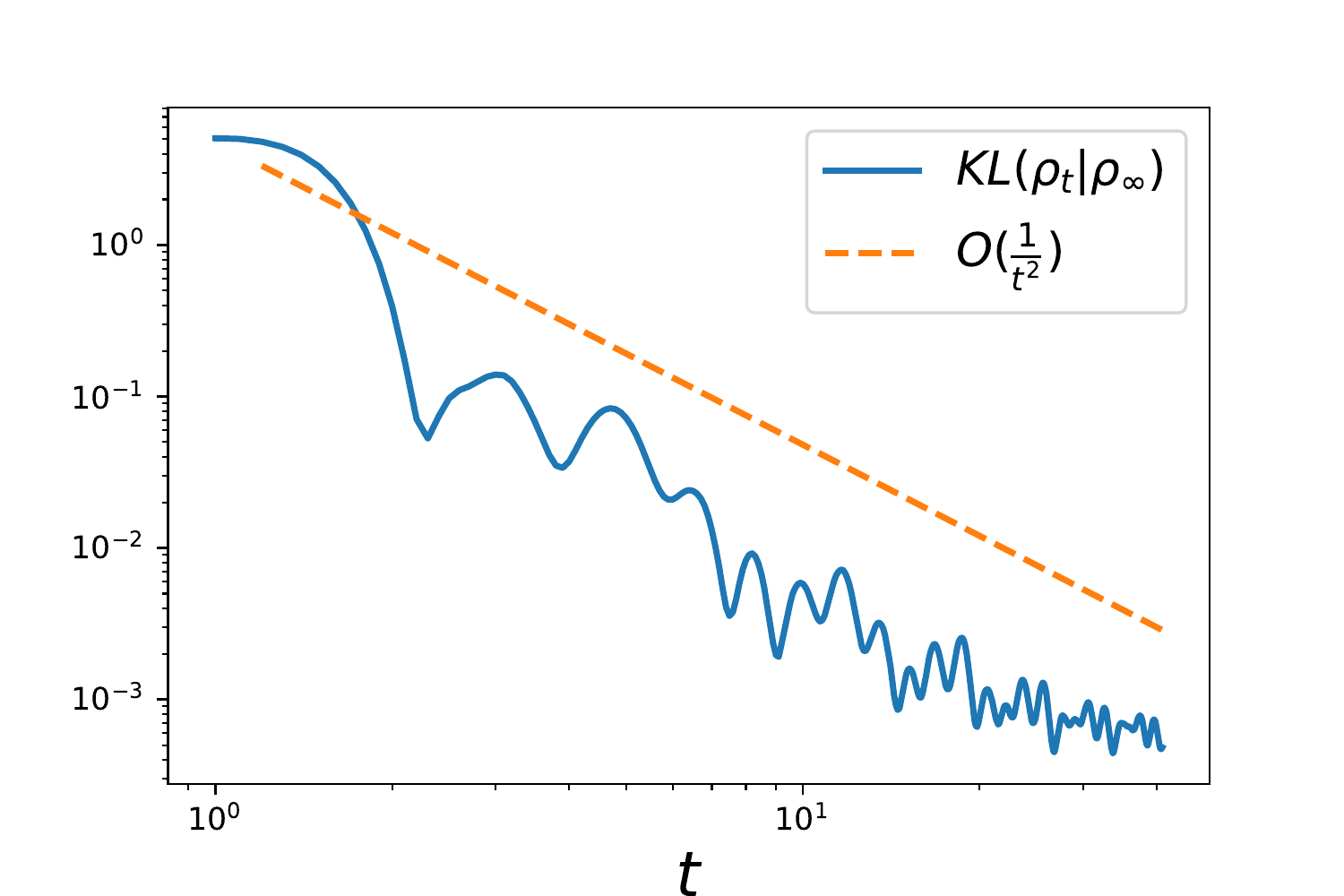}
			
		}		
	\end{tabular}
	\caption{  
		Simulation result for the non-Gaussian case (Example~\ref{sec:num-non-Gaussian}): (a) The time
		traces of the particles; (b) The KL-divergence as a
		function of time.}
	\label{fig:non-Gaussian-res}
\end{figure}

	\subsection{Non-Gaussian example}\label{sec:num-non-Gaussian}
	This example involves a non-Gaussian target distribution $\pr_\infty =
	\frac{1}{2} \mathcal{N}(-m,\sigma^2) +
	\frac{1}{2}\mathcal{N}(m,\sigma^2)$ which is a mixture of two
	one-dimensional Gaussians with $m=2.0$ and $\sigma^2=0.8$.  The
	simulation results are depicted in
	Figure~\ref{fig:non-Gaussian-res}-(a)-(b).  The numerical parameters are
	same as in the Example~\ref{sec:num-Gaussian}.  The interaction term is
	approximated using the diffusion map approximation with $\epsilon=0.01$.  The numerical result depicted in Figure~\ref{fig:non-Gaussian-res}-(a) show that the diffusion map algorithm converges to the mixture of Gaussian target distribution.
	The result depicted in Figure~\ref{fig:non-Gaussian-res}-(b) suggests that the convergence rate $O(e^{-\beta_t})$ also appears to hold for this non-log-concave target distribution. Theoretical justification of this is subject of continuing work.


	\subsection{Comparison with MCMC and HMCMC}
	This section contains numerical experiment comparing the performance of the accelerated algorithm~\ref{alg:Gaussian-kernel} using the diffusion map (DM) approximation~\eqref{eq:interaction-kernel} and the density estimation~(DE)-based approximation ~\eqref{eq:density-estimation} with the Markov chain Monte-Carlo (MCMC) algorithm studied in~\cite{durmus2016high}  and the Hamiltonian MCMC algorithm studied in~\cite{cheng2017underdamped}.  
	
	We consider the problem setting of the mixture of Gaussians as in example~\ref{sec:num-non-Gaussian}. All algorithms are simulated with a fixed step-size of $\Delta t = 0.1$ for $K=1000$ iterations. The performance is measured by computing the  mean-squared error in estimating the expectation of the function $\psi(x)=x{1}_{x\geq 0}$. The mean-square error at the $k$-th iteration is computed by averaging the error over $M=100$ runs:
	\begin{equation}\label{eq:mse-error}
	\text{m.s.e}_k=\frac{1}{M}\sum_{m=1}^M\left(\frac{1}{N}\sum_{i=1}^N \psi(X^{i,m}_{t_k})-\int \psi(x)\pr_\infty(x)\ud x\right)^2
	\end{equation}
	
	The numerical results are depicted in Figure~\ref{fig:mixture-kernel-comparison-N}. Figure~\ref{fig:comparison-N} depicts the m.s.e as a function of $N$. It is observed that the accelerated algorithm~\ref{alg:Gaussian-kernel} with the diffusion map approximation admits an order of magnitude better m.s.e for the same number of particles. 
	It is also observed that the m.s.e decreases rapidly for intermediate values of $N$ before saturating for large values of $N$, where the bias term dominates~(see discussion following Eq.~\ref{eq:interaction-kernel}).
	
	Figure~\ref{fig:comparison-K} depicts the m.s.e as a function of the number of iterations for a fixed number of particles $N=100$. It is observed that the accelerated algorithm~\ref{alg:Gaussian-kernel} displays the quickest convergence amongst the algorithms tested. 
	  
	 Figure~\ref{fig:comparison-time} depicts the average computational time per iteration as a function of the number of samples $N$. The computational time of the diffusion map approximation scales as $O(N^2)$ because it involves computing a $N \times N$ matrix $[\keps(X^i,X^j)]_{i,j=1}^N$, while the computational cost of the MCMC and HMCMC algorithms scale as $O(N)$. 
	 The computational complexity may be improved by (i) exploiting the sparsity structure of the $N\times N$ matrix ; (ii) sub-sampling the particles in computing the empirical averages; (iii) adaptively updating the $N\times N$ matrix according to a certain error criteria. 
	 	
	Finally, we provide comparison between diffusion map approximation \eqref{eq:density-estimation} and the density-based approximation ~\eqref{eq:density-estimation}: Figure~\ref{fig:comparison-eps} depicts the m.s.e for these two approximations as a function of the kernel-bandwidth $\epsilon$ for  a fixed number of particles $N=100$. For very large and for very small values of $\epsilon$, where bias and variance dominates the error, respectively, the two algorithms have similar m.s.e. However, for intermediate values of $\epsilon$, the diffusion map approximation has smaller variance, and thus lower m.s.e.

	\section{Conclusion and directions for future work}\label{sec:conclusion}
	The main contribution of this paper is to extend the variational
	formulation of~\cite{wibisono2016} to obtain theoretical results and
	numerical algorithms for accelerated gradient flow in the space of
	probability distributions.  In continuous-time settings, bounds on
	convergence rate are derived based on a Lyapunov function argument.
	Two numerical algorithms based upon an interacting particle
	representation are presented and illustrated with examples. As has been the case in
	finite-dimensional settings, the theoretical framework is expected to
	be useful in this regard.  Some direction for future include: (i) removing the technical assumption in the proof of the Theorem~\ref{thm:main-res}; (ii) analysis of the convergence under the weaker assumption that the target distribution satisfies only a spectral gap condition; and (iii)  analysis
	of the numerical algorithms in the finite-$N$ and in the finite $\Delta t$ cases.

	
		 	\begin{figure}[H]
		\begin{tabular}{cc}
			\subfigure[]{
				\includegraphics[width=0.5\columnwidth]{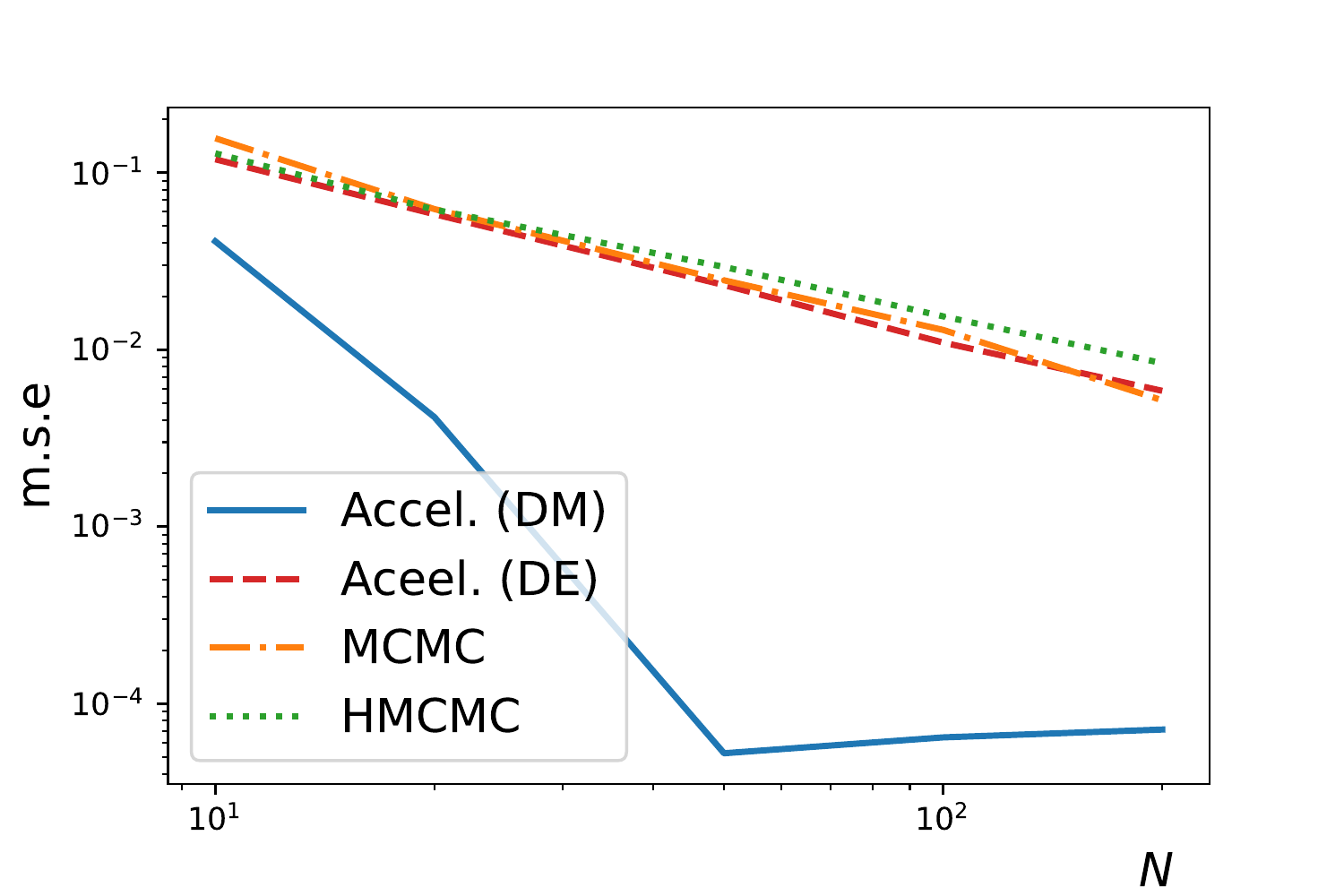}\label{fig:comparison-N}
				
			}&
			\subfigure[]{
				\includegraphics[width=0.5\columnwidth]{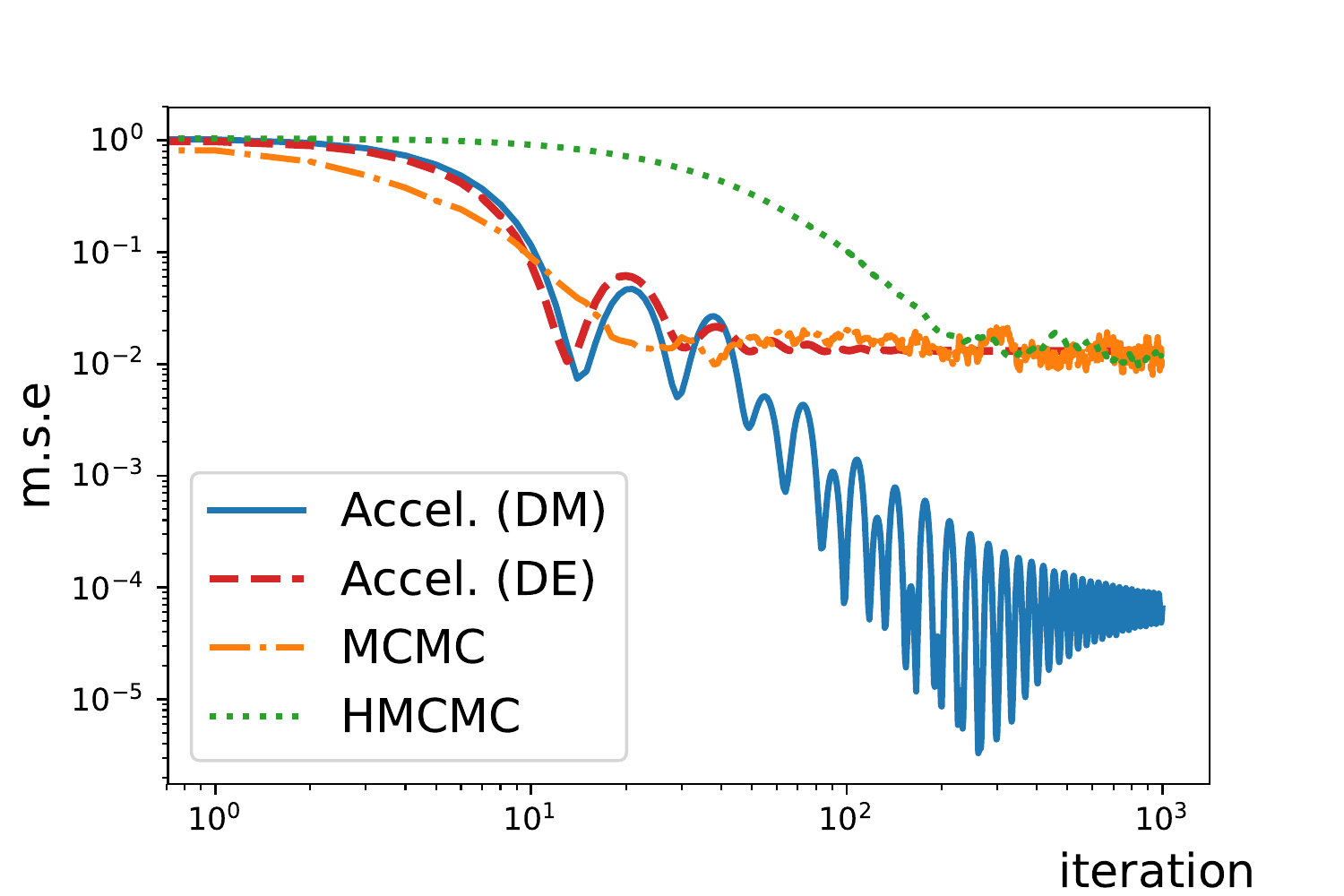}
				\label{fig:comparison-K}
			}\\
			\subfigure[]{
				\includegraphics[width=0.5\columnwidth]{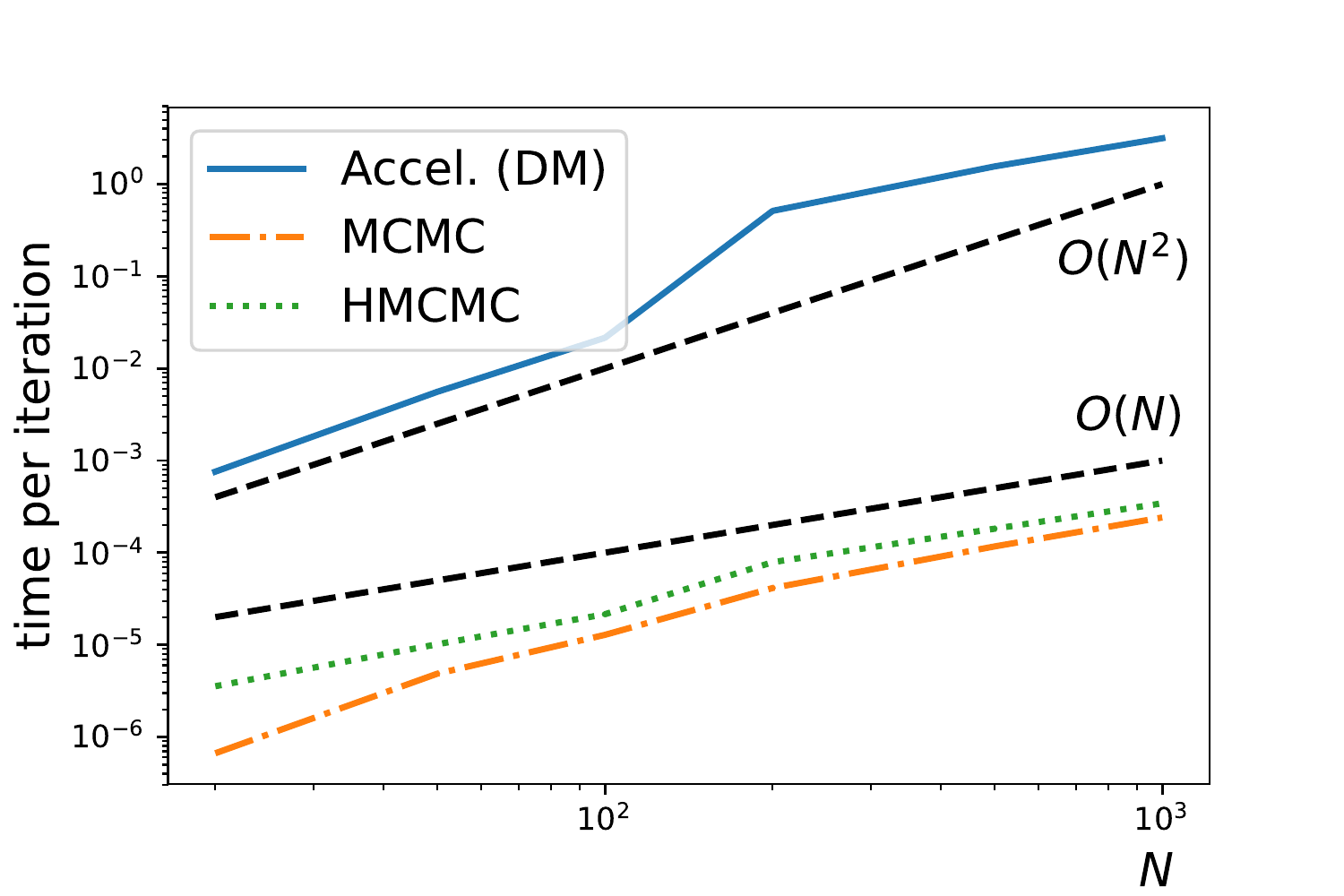}
				\label{fig:comparison-time}
			}&
			\subfigure[]{
				\includegraphics[width=0.5\columnwidth]{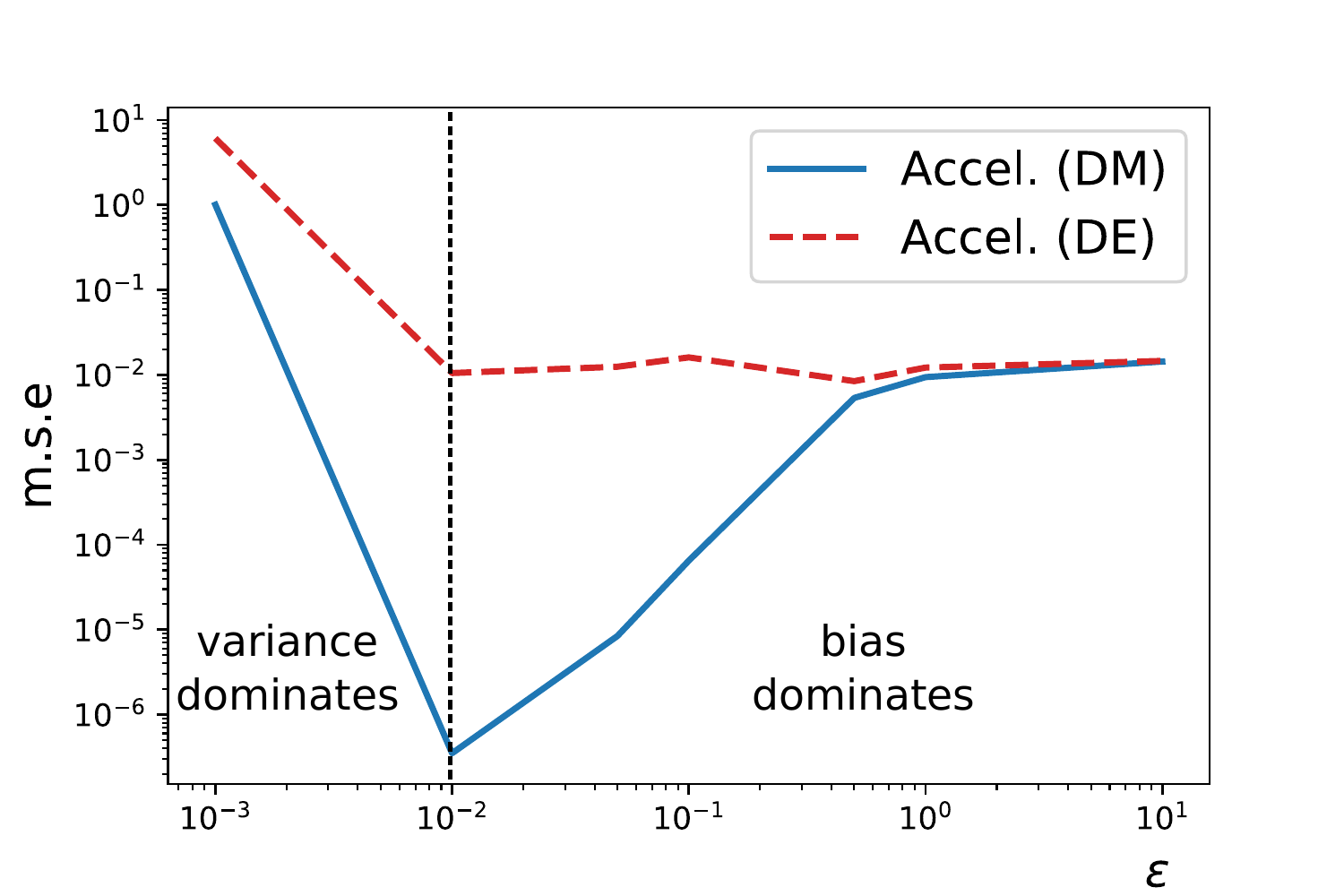}
				\label{fig:comparison-eps}
			}
		\end{tabular}
		\caption{Simulation-based comparison of the performance of the accelerated algorithm~\ref{alg:Gaussian-kernel} using the diffusion map~(DM) approximation~\eqref{eq:interaction-kernel}, the density estimation~(DE)-based approximation~\eqref{eq:density-estimation} with the  MCMC and HMCMC algorithms: (a) the mean-squared error (m.s.e)~\eqref{eq:mse-error} as a function of the number of samples $N$; (b) the m.s.e as a function of the number of iterations; (c) the average computational time per iteration as a function of the number of samples; (d) m.s.e comparison between the diffusion map and the density estimation-based approaches as a function of the kernel bandwidth $\epsilon$.}
		\label{fig:mixture-kernel-comparison-N}
	\end{figure}
	
	\bibliographystyle{plain}
	\bibliography{optimization}

\begin{thebibliography}{10}

\bibitem{ambrosio2008gradient}
Luigi Ambrosio, Nicola Gigli, and Giuseppe Savar{\'e}.
\newblock {\em Gradient flows: in metric spaces and in the space of probability
  measures}.
\newblock Springer Science \& Business Media, 2008.

\bibitem{arjovsky2017wasserstein}
Martin Arjovsky, Soumith Chintala, and L{\'e}on Bottou.
\newblock Wasserstein gan.
\newblock {\em arXiv preprint arXiv:1701.07875}, 2017.

\bibitem{betancourt2018symplectic}
Michael Betancourt, Michael~I Jordan, and Ashia~C Wilson.
\newblock On symplectic optimization.
\newblock {\em arXiv preprint arXiv:1802.03653}, 2018.

\bibitem{blei2017variational}
David~M Blei, Alp Kucukelbir, and Jon~D McAuliffe.
\newblock Variational inference: A review for statisticians.
\newblock {\em Journal of the American Statistical Association},
  112(518):859--877, 2017.

\bibitem{carmona2017probabilistic}
Rene Carmona and Fran{\c{c}}ois Delarue.
\newblock {\em Probabilistic Theory of Mean Field Games with Applications
  I-II}.
\newblock Springer, 2017.

\bibitem{chen2018unified}
Changyou Chen, Ruiyi Zhang, Wenlin Wang, Bai Li, and Liqun Chen.
\newblock A unified particle-optimization framework for scalable bayesian
  sampling.
\newblock {\em arXiv preprint arXiv:1805.11659}, 2018.

\bibitem{cheng2017underdamped}
Xiang Cheng, Niladri~S Chatterji, Peter~L Bartlett, and Michael~I Jordan.
\newblock Underdamped langevin mcmc: A non-asymptotic analysis.
\newblock {\em arXiv preprint arXiv:1707.03663}, 2017.

\bibitem{chizat2018global}
Lenaic Chizat and Francis Bach.
\newblock On the global convergence of gradient descent for over-parameterized
  models using optimal transport.
\newblock {\em arXiv preprint arXiv:1805.09545}, 2018.

\bibitem{coifman2006diffusion}
Ronald~R Coifman and St{\'e}phane Lafon.
\newblock Diffusion maps.
\newblock {\em Applied and computational harmonic analysis}, 21(1):5--30, 2006.

\bibitem{durmus2016high}
Alain Durmus and Eric Moulines.
\newblock High-dimensional bayesian inference via the unadjusted langevin
  algorithm.
\newblock {\em arXiv preprint arXiv:1605.01559}, 2016.

\bibitem{evensen2003ensemble}
Geir Evensen.
\newblock The ensemble kalman filter: Theoretical formulation and practical
  implementation.
\newblock {\em Ocean dynamics}, 53(4):343--367, 2003.

\bibitem{frogner2018approximate}
Charlie Frogner and Tomaso Poggio.
\newblock Approximate inference with wasserstein gradient flows.
\newblock {\em arXiv preprint arXiv:1806.04542}, 2018.

\bibitem{goodfellow2014generative}
Ian Goodfellow, Jean Pouget-Abadie, Mehdi Mirza, Bing Xu, David Warde-Farley,
  Sherjil Ozair, Aaron Courville, and Yoshua Bengio.
\newblock Generative adversarial nets.
\newblock In {\em Advances in neural information processing systems}, pages
  2672--2680, 2014.

\bibitem{hein2007graph}
Matthias Hein, Jean-Yves Audibert, and Ulrike~von Luxburg.
\newblock Graph laplacians and their convergence on random neighborhood graphs.
\newblock {\em Journal of Machine Learning Research}, 8(Jun):1325--1368, 2007.

\bibitem{jain2017accelerating}
Prateek Jain, Sham~M Kakade, Rahul Kidambi, Praneeth Netrapalli, and Aaron
  Sidford.
\newblock Accelerating stochastic gradient descent.
\newblock {\em arXiv preprint arXiv:1704.08227}, 2017.

\bibitem{jordan1998variational}
Richard Jordan, David Kinderlehrer, and Felix Otto.
\newblock The variational formulation of the fokker--planck equation.
\newblock {\em SIAM journal on mathematical analysis}, 29(1):1--17, 1998.

\bibitem{liu2016stein}
Qiang Liu and Dilin Wang.
\newblock Stein variational gradient descent: A general purpose bayesian
  inference algorithm.
\newblock In {\em Advances In Neural Information Processing Systems}, pages
  2378--2386, 2016.

\bibitem{mccann1997convexity}
Robert~J McCann.
\newblock A convexity principle for interacting gases.
\newblock {\em Advances in mathematics}, 128(1):153--179, 1997.

\bibitem{neal2011mcmc}
Radford~M Neal et~al.
\newblock Mcmc using hamiltonian dynamics.
\newblock {\em Handbook of Markov Chain Monte Carlo}, 2(11):2, 2011.

\bibitem{richemond2017wasserstein}
Pierre~H Richemond and Brendan Maginnis.
\newblock On wasserstein reinforcement learning and the fokker-planck equation.
\newblock {\em arXiv preprint arXiv:1712.07185}, 2017.

\bibitem{su2014}
Weijie Su, Stephen Boyd, and Emmanuel Candes.
\newblock A differential equation for modeling nesterov’s accelerated
  gradient method: Theory and insights.
\newblock In {\em Advances in Neural Information Processing Systems}, pages
  2510--2518, 2014.

\bibitem{sutton2000policy}
Richard~S Sutton, David~A McAllester, Satinder~P Singh, and Yishay Mansour.
\newblock Policy gradient methods for reinforcement learning with function
  approximation.
\newblock In {\em Advances in neural information processing systems}, pages
  1057--1063, 2000.

\bibitem{wibisono2016}
Andre Wibisono, Ashia~C Wilson, and Michael~I Jordan.
\newblock A variational perspective on accelerated methods in optimization.
\newblock {\em Proceedings of the National Academy of Sciences}, page
  201614734, 2016.

\bibitem{zhang2018policy}
Ruiyi Zhang, Changyou Chen, Chunyuan Li, and Lawrence Carin.
\newblock Policy optimization as wasserstein gradient flows.
\newblock {\em arXiv preprint arXiv:1808.03030}, 2018.

\end{thebibliography}

	\appendix
	\section{PDE formulation of the variational problem }\label{apdx:pde}
	An equivalent pde formulation is obtained by
	considering the stochastic optimal control
	problem~\eqref{eq:var-problem-X} as a deterministic optimal control
	problem on the space of the probability distributions. 
	Specifically, the process $\{\rho_t\}_{t\geq 0}$ is a deterministic
	process that takes values in $\Pspace(\Re^d)$ and evolves according to the continuity equation
	\begin{equation*}
	\frac{\partial \rho_t}{\partial t} = -\nabla \cdot (\rho_t u_t)
	\end{equation*}
	where $u_t:\Re^d \to \Re^d$ is now a time-varying vector field. 
	The Lagrangian $\LP:\Re^+\times \Pspace(\Re^d) \times L^2(\Re^d;\Re^d) \to \Re$ is defined as:
	\begin{equation}
	\LP(t,\pr,u):= e^{\alpha_t+\gamma_t}\left[\int_{\Re^d}
	\frac{1}{2}|e^{-\alpha_t}u(x)|^2\pr(x) \ud x -e^{\beta_t}{\sf F}(\pr)\right]
	\label{eq:Lagrangian-p}
	\end{equation}
	The optimal control problem is:
	\begin{equation}
	\begin{aligned}
	\text{Minimize}\quad& \int_0^\infty \LP (t,\rho_t,u_t) \ud t \\
	\text{Subject to}\quad& \frac{\partial \pr_t}{\partial t} + \nabla \cdot (\pr_t u_t) = 0
	\end{aligned}
	\label{eq:var-problem-p}
	\end{equation}
	
	The Hamiltonian function $\HP:\Re^+\times \Pspace(\Re^d)\times \fSpace \times L^2(\Re^d;\Re^d) \to \Re$ is 
	\begin{equation}\label{eq:Hamilton-function-p}
	\HP (t,\rho,\phi,u):= \lr{\nabla \phi}{u}_{L^2(\pr)}- \LP(t,\rho,u)
	\end{equation}
	where  $\phi \in \fSpace$ is the dual variable and  the inner-product $\lr{\nabla \phi}{u}_{L^2(\pr)}:=\int_{\Re^d} \nabla \phi(x)\cdot u(x)\pr(x)\ud x$
	
	\section{Restatement of the main result and its proof} \label{apdx:proof}
	We restate Theorem~\ref{thm:main-res} below which now includes the pde formulation as well. 
	\begin{theorem}\label{thm:main-res-2}
		Consider the variational problem
		\eqref{eq:var-problem-X}-\eqref{eq:var-problem-p}.  
		\begin{romannum}
			\item For the probabilistic form~\eqref{eq:var-problem-X} of the
			variational problem, the optimal control $U_t^* =
			e^{\alpha_t-\gamma_t}Y_t$, where the optimal trajectory $\{(X_t,Y_t)\}_{t \geq 0}$
			evolves according to the Hamilton's odes:
			\begin{subequations}				\label{eq:Hamilton-X}
				\begin{align}
				\frac{\ud X_t}{\ud t}  &= U_t^* =
				e^{\alpha_t-\gamma_t}Y_t,\quad
				X_0 \sim \rho_0\\
				\frac{\ud Y_t}{\ud t}  &=  - e^{\alpha_t+\beta_t+\gamma_t}\nabla_\pr F(\pr_t)(X_t),\quad Y_0=\nabla \phi_0(X_0)
				\end{align}
			\end{subequations}
			where $\phi_0$ is a convex function, and $\pr_t = \text{Law}(X_t)$.
			\vspace{-10pt}
			\item For the pde form~\eqref{eq:var-problem-p} of the variational problem, the optimal control is $u_t^*
			= e^{\alpha_t-\gamma_t} \nabla \phi_t(x)$, where the
			optimal trajectory $\{(\rho_t,\phi_t)\}_{t \geq 0}$
			evolves according to the Hamilton's pdes:
			\begin{subequations}\label{eq:Hamilton-p}
				\begin{align}
				\frac{\partial \pr_t}{\partial t} &= -\nabla \cdot (\pr_t
				\underbrace{e^{\alpha_t-\gamma_t}
					\nabla \phi_t}_{u_t^*}
				),\quad \text{\it initial
					condn.} \;\;\rho_0 \label{eq:acc-flow-p}\\
				\frac{\partial \phi_t}{\partial t} &= -e^{\alpha_t-\gamma_t} \frac{|\nabla \phi_t|^2}{2} - e^{\alpha_t + \gamma_t + \beta_t}\nabla_{\pr}F(\pr)\label{eq:acc-flow-phi}
				\end{align}
			\end{subequations}
			\item The solutions of the two forms are equivalent in the
			following sense:
			\begin{equation*}
			\text{Law}(X_t) = \rho_t,\quad U_t=u_t(X_t),\quad Y_t= \nabla \phi_t(X_t)
			\end{equation*}	
			\item Suppose additionally that the functional $F$ is displacement convex and $\pr_\infty$ is its minimizer. Define
			\begin{equation}\label{eq:Lyapunov-function-2}
			V(t) = \frac{1}{2}\Expect (|X_t + e^{-\gamma_t}Y_t -
			T_{\pr_t}^{\pr_\infty}(X_t)|^2) + e^{\beta_t}(F(\pr)-F(\pr_\infty))
			\end{equation}
			where the map $T_{\pr_t}^{\pr_\infty}:\Re^d\to \Re^d$ is the optimal transport map from $\pr_t$ to $\pr_\infty$.  Suppose also that the following technical assumption holds: $\Expect[(X_t+e^{-\gamma_t}Y_t -
			T^{\pr_\infty}_{\pr_t}(X_t))\cdot \frac{\ud}{\ud t} T^{\pr_\infty}_{\pr_t}(X_t)]=0$. Then $\frac{\ud V}{\ud t}(t) \leq 0$. Consequently, the following rate of convergence  is obtained along the optimal trajectory
			\begin{equation*}
			F(\pr_t)-F(\pr_\infty) \leq O(e^{-\beta_t}),\quad \forall  t\geq 0
			\end{equation*}
		\end{romannum}
		%
	\end{theorem}
	\begin{proof}  
		\begin{romannum}
			\item  The Hamiltonian function defined in~\eqref{eq:Hamilton-function-X} is equal to
			\begin{equation*}
			\HX(t,x,\pr,y,u) = y \cdot u - e^{\gamma_t-\alpha_t} \frac{1}{2}|u|^2+e^{\alpha_t+\gamma_t\beta_t}\tilde{F}(\pr,x)
			\end{equation*}		
			after inserting the formula for the Lagrangian.
			According to the maximum principle in probabilistic
			form for (mean-field) optimal control problems~(see \cite[Sec. 6.2.3]{carmona2017probabilistic}),
			the optimal control law $U^*_t=\argmin_v \HX(t,X_t,\rho_t,Y_t,v)=e^{\alpha_t - \gamma_t}Y_t $ and the Hamilton's equations are
			\begin{align*}
			\frac{\ud {X}_t}{\ud t} &= +\nabla_y \HX (t,X_t,\rho_t,Y_t,U^*_t)=U^*_t= e^{\alpha_t - \gamma_t}Y_t\\
			\frac{\ud {Y}_t}{\ud t} &= -\nabla_x \HX (t,X_t,\rho_t,Y_t,U^*_t) -  \tilde{\Expect}[\nabla_\rho \HX(t,\tilde{X}_t,\rho_t,\tilde{Y}_t,\tilde{U}^*_t)(X_t)]
			\end{align*}
			where $\tilde{X}_t,\tilde{Y_t},\tilde{U}^*_t$ are independent copies of $X_t,Y_t,U_t^*$. The derivatives  
			\begin{align*}
			\nabla_x \HX (t,x,\rho,y,u) &=
			e^{\alpha_t+\beta_t+\gamma_t}\nabla_x
			\tilde{F} (\pr,x)\\
			\nabla_\pr \HX (t,x,\rho,y,u) &= e^{\alpha_t+\beta_t+\gamma_t}\nabla_\pr \tilde{F}(\pr,x)		
			\end{align*}
			It follows that 
			\begin{align*}
			\frac{\ud {Y}_t}{\ud t} =
			-e^{\alpha_t+\beta_t+\gamma_t}\left(\nabla_x
			\tilde{F} (\pr_t,X_t)  + \tilde{\Expect}[\nabla_\pr \tilde{F}(\pr_t,\tilde{X}_t)(X_t)]\right) =  -e^{\alpha_t+\beta_t+\gamma_t}\nabla_\pr {\sf F}(\pr)(X_t)
			\end{align*}
			where we used the definition ${\sf F}(\pr) = \int
			\tilde{F}(x,\pr)\pr(x)\ud x$ and the
			identity~\cite[Sec. 5.2.2 Example
			3]{carmona2017probabilistic} \[\nabla_\pr {\sf F}(\pr)(x) =
			\nabla_x \tilde{F} (\pr,x) + \int \nabla_\pr \tilde{F}(\pr,\tilde{x})(x)\pr(\tilde{x})\ud \tilde{x}\]
			%
			\item The Hamiltonian function defined in~\eqref{eq:Hamilton-function-p} is equal to
			\begin{equation*}
			{\HP}(t,\pr,\phi,u) = \int \left[\nabla \phi (x) \cdot u(x) -\frac{1}{2}e^{\gamma_t-\alpha_t}|u(x)|^2\right]\pr(x)\ud x  + e^{\alpha_t+\gamma_t+\beta_t}{\sf F}(\pr)
			\end{equation*}	
			after inserting the formula for the Lagrangian.
			According to the maximum principle for pde formulation of mean-field optimal control problems~(see \cite[Sec. 6.2.4]{carmona2017probabilistic}) the optimal control vector field is $u^*_t = \argmin_v {\HP}(t,\rho_t,\phi_t,v)=e^{\alpha_t-\gamma_t}\nabla \phi_t$ and the Hamilton's equations are:
			\begin{align*}
			\frac{\partial \rho_t}{\partial t} &= +\frac{\partial \HP}{\partial \phi}(t,\rho_t,\phi_t,u_t) = -\nabla \cdot(\rho_t \nabla u^*_t) \\
			\frac{\partial \phi_t}{\partial t} &= -\frac{\partial \HP}{\partial \pr}(t,\rho_t,\phi_t,u_t) = -(\nabla \phi \cdot u^* - e^{\gamma_t-\alpha_t}\frac{1}{2}|u^*_t|^2 + e^{\alpha_t+\gamma_t+\beta_t} \frac{\partial {\sf F}}{\partial \pr}(\pr_t))
			\end{align*}
			inserting the formula $u^*_t=e^{\alpha_t-\gamma_t}\nabla \phi_t$ concludes the result.
			\item Consider the $(\pr_t,\phi_t)$ defined from~\eqref{eq:Hamilton-p}. The distribution $\pr_t$ is identified with a stochastic process $\tilde{X}_t$ such that $\frac{\ud \tilde{X}_t}{\ud t}=e^{\alpha_t-\gamma_t} \nabla \phi_t(\tilde{X}_t)$ and $\text{Law}(\tilde{X}_t)=\pr_t$. Then define $\tilde{Y}_t = \nabla \phi_t(\tilde{X}_t)$. Taking the time derivative shows that 
			\begin{align*}
			\frac{\ud \tilde{Y}_t}{\ud t} &= \frac{\ud}{\ud t}\nabla \phi_t(\tilde{X}_t)=\nabla^2 \phi_t(\tilde{X}_t) \frac{\ud \tilde{X}_t}{\ud t} + \nabla \frac{\partial \phi_t}{\partial t}(X_t)\\&=e^{\alpha_t-\gamma_t}\nabla^2 \phi_t(\tilde{X}_t)\nabla \phi_t(\tilde{X}_t) - e^{\alpha_t-\gamma_t}\nabla^2 \phi_t(\tilde{X}_t)\nabla \phi_t(X_t) - e^{\alpha_t+\beta_t+\gamma_t} \nabla \frac{\partial{\sf F}}{\partial \pr}(\pr_t)(\tilde{X}_t)\\
			&=- e^{\alpha_t+\beta_t+\gamma_t} \nabla \frac{\partial {\sf F}}{\partial \pr}(\pr_t)(\tilde{X}_t)\\
			&=- e^{\alpha_t+\beta_t+\gamma_t} \nabla_\pr {\sf F}(\pr_t)(\tilde{X}_t)
			\end{align*}
			with the initial condition $\tilde{Y}_0=\nabla \phi_0(\tilde{X}_0)$, where we used the identity $\nabla_x \frac{\partial {\sf F}}{\partial \pr}(\pr) = \nabla_\pr {\sf F}(\pr)$~\cite[Prop. 5.48]{carmona2017probabilistic}. Therefore the equations for $\tilde{X}_t$ and $\tilde{Y}_t$ are identical. Hence one can identify $(X_t,Y_t)$ with $(\tilde{X}_t,\tilde{Y}_t)$.
			\item The energy functional 
			\begin{equation*}
			V(t) = \underbrace{\frac{1}{2}\Expect \left[|X_t + e^{-\gamma_t}Y_t - T_{\pr_t}^{\pr_\infty}(X_t)|^2\right]}_{\text{first term}} + \underbrace{e^{\beta_t}({\sf F}(\pr)-{\sf F}(\pr_\infty))}_{\text{second term}}
			\end{equation*}	
			Then the derivative of the first term is
			\begin{align*}
			\Expect\left[(X_t+e^{-\gamma_t}Y_t - T^{\pr_\infty}_{\pr_t}(X_t))\cdot (e^{\alpha_t-\gamma_t}Y_t - \dot{\gamma}_te^{-\gamma_t}Y_t - e^{\alpha_t+\beta_t}\nabla_\pr \FP(\pr_t)(X_t) + \xi(T^{\pr_\infty}_{\pr_t}(X_t)))\right]
			\end{align*}
			where $\xi(T^{\pr_\infty}_{\pr_t}(X_t)):=\frac{\ud}{\ud t}T^{\pr_\infty}_{\pr_t}(X_t)$. Using the scaling condition $\dot{\gamma}_t=e^{\alpha_t}$ the derivative of the first term simplifies to
			\begin{align*}
			\Expect\left[(X_t+e^{-\gamma_t}Y_t - T^{\pr_\infty}_{\pr_t}(X_t))\cdot(- e^{\alpha_t+\beta_t}\nabla_\pr {\sf F}(\pr_t)(X_t) + \xi(T^{\pr_\infty}_{\pr_t}(X_t)))\right]
			\end{align*}		
			Upon using the technical assumption,
			$\Expect[(X_t+e^{-\gamma_t}Y_t -
			T^{\pr_\infty}_{\pr_t}(X_t))\cdot \xi(T^{\pr_\infty}_{\pr_t}(X_t))]=0$ the derivative of the first term simplifies to
			\begin{align*}
			\Expect\left[(X_t+e^{-\gamma_t}Y_t - T^{\pr_\infty}_{\pr_t}(X_t))\cdot(- e^{\alpha_t+\beta_t}\nabla_\pr F(\pr_t)(X_t))\right]
			\end{align*}
			The derivative of the second term is
			\begin{align*}
			\frac{\ud}{\ud t}(\text{second term})&=\dot{\beta}_t e^{\beta_t}({\sf F}(\pr_t)-{\sf F}(\pr_\infty)) + e^{\beta_t}\frac{\ud}{\ud t} {\sf F}(\pr_t)\\
			&=e^{\alpha_t+\beta_t}({\sf F}(\pr_t)-{\sf F}(\pr_\infty)) + e^{\beta_t}\Expect[\nabla_\pr {\sf F}(\pr_t)(X_t)e^{\alpha_t-\gamma_t}Y_t]
			\end{align*}
			where we used the scaling condition $\dot{\beta_t}=e^{\alpha_t}$ and the chain-rule for the Wasserstein gradient~\cite[Ch. 10, E. Chain rule]{ambrosio2008gradient}. Adding the derivative of the first and second term yields:
			\begin{align*}
			\frac{\ud V}{\ud t}(t)=e^{\alpha_t+\beta_t}\left({\sf F}(\pr_t)-{\sf F}(\pr_\infty) - \Expect\left[(X_t- T^{\pr_\infty}_{\pr_t}(X_t))\cdot \nabla_\pr {\sf F}(\pr_t)(X_t)\right]\right)
			\end{align*}
			which is negative by variational inequality characterization of the displacement convex function ${\sf F}(\pr)$~\cite[Eq. 10.1.7]{ambrosio2008gradient}.
			
			We expect that the technical assumption can be
			removed.  This is the subject of the continuing work.

		\end{romannum}
	\end{proof}

\section{Wasserstein gradient and Gâteaux derivative}\label{apdx:derivative}
This section contains definitions of the Wasserstein gradient and Gâteaux derivative~\cite{ambrosio2008gradient,carmona2017probabilistic}.  

Let $\FP:\Pspace(\Re^d) \to \Re$ be a (smooth) functional on the space of probability distributions. 

\newP{Gâteaux derivative} The Gâteaux derivative of $\FP$ at $\pr \in \Pspace(\Re^d)$ is a real-valued function on $\Re^d$ denoted as $\frac{\partial \FP}{\partial \pr}(\pr):\Re^d\to\Re$. It is  defined as a function that satisfies the identity
\begin{equation*}
\frac{\ud}{\ud t} \FP(\pr_t) \bigg\vert_{t=0}= \int_{\Re^d} \frac{\partial \FP}{\partial \pr}(\pr)(x) (-\nabla \cdot(\pr(x)u(x))) \ud x
\end{equation*} 
for all path $\pr_t$ in $\Pspace(\Re^d)$ such that $\frac{\partial \pr_t}{\partial t}=-\nabla \cdot(\pr_t u)$ with $\pr_0=\pr\in \Pspace(\Re^d)$.

\newP{Wasserstein gradient} The Wasserstein gradient of $\FP$ at $\pr$  is a vector-field on $\Re^d$ denoted as $\nabla_\pr \FP(\pr):\Re^d \to \Re^d$. It is defined as a vector-field that satisfies the identity  
\begin{equation*}
\frac{\ud}{\ud t} \FP(\pr_t) \bigg\vert_{t=0}= \int_{\Re^d} \nabla_\pr \FP(\pr)(x) \cdot u(x)~\pr(x)\ud x
\end{equation*} 
for all path $\pr_t$ in $\Pspace(\Re^d)$ such that $\frac{\partial \pr_t}{\partial t}=-\nabla \cdot(\pr_t u)$ with $\pr_0=\pr\in \Pspace(\Re^d)$.

The two definitions imply the following relationship~\cite[Prop. 5.48]{carmona2017probabilistic}:
\begin{equation*}
\nabla_\pr \FP(\pr)(\cdot) = \nabla_x \frac{\partial\FP}{\partial \pr } (\pr)(\cdot)
\end{equation*}

\newP{Example} Let $\FP(\pr)  = \int \log(\frac{\pr(x)}{\pr_\infty(x)})\pr(x)\ud x$ be the relative entropy functional. Consider a path  $\pr_t$ in $\Pspace(\Re^d)$ such that $\frac{\partial \pr_t}{\partial t}=-\nabla \cdot(\pr_t u)$ with $\pr_0=\pr\in \Pspace(\Re^d)$. Then
\begin{align*}
\frac{\ud}{\ud t} \FP(\pr_t)  &= \int \log(\frac{\pr_t(x)}{\pr_\infty(x)})\frac{\partial \pr_t}{\partial t}(x)\ud x  + \int \frac{\partial \pr_t}{\partial t}(x)\ud x\\&=  -\int \log(\frac{\pr_t(x)}{\pr_\infty(x)}) \nabla \cdot (\pr_t(x) u(x))\ud x\\
&= \int \nabla_x \log(\frac{\pr_t(x)}{\pr_\infty(x)})  \cdot u(x) ~ \pr_t(x) \ud x
\end{align*}
where the divergence theorem is used in the last step. The  definitions of the Gâteaux derivative and Wasserstein gradient imply 
\begin{align*}
\frac{\partial \FP}{\partial \pr}(\pr)(x) &= \log (\frac{\pr(x)}{\pr_\infty(x)})
\\
\nabla_\pr \FP(\pr)(x) & = \nabla_x \log (\frac{\pr(x)}{\pr_\infty(x)})
\end{align*}

\section{Relationship with the under-damped Langevin equation} \label{apdx:Langevin}
A basic form of the under-damped (or second order) Langevin equation is given in~\cite{cheng2017underdamped}
\begin{equation}
\begin{aligned}
\ud X_t &= v_t\ud t\\
\ud v_t &= - \gamma v_t\ud t  - \nabla f(X_t)\ud t +\sqrt{2}\ud B_t
\end{aligned}
\label{eq:under-dampled-Langevin}
\end{equation} 
where $\{B_t\}_{t\geq 0}$ is the standard Brownian motion. 

Consider next, the the accelerated flow~\eqref{eq:Hamilton-relative-entropy}. Denote $v_t := e^{\alpha_t-\gamma_t}Y_t$. Then, with an appropriate choice of scaling parameters (e.g. $\alpha_t=0$, $\beta_t=0$ and $\gamma_t=-\gamma t$ ): 
\begin{equation}
\begin{aligned}
\ud X_t &= v_t\ud t\\
\ud v_t &= - \gamma v_t\ud t  - \nabla f(X_t)\ud t - \nabla_x \log(\pr_t(X_t))
\end{aligned}
\label{eq:acc-flow-comparison}
\end{equation}

The scaling parameters are chosen here for the sake of comparison and do not satisfy the ideal scaling conditions of~\cite{wibisono2016}. 

The sdes~\eqref{eq:under-dampled-Langevin} and~\eqref{eq:acc-flow-comparison} are similar except that the stochastic term $\sqrt{2}\ud B_t$ in~\eqref{eq:under-dampled-Langevin} is replaced with a deterministic term  $-\nabla_x \log(\pr_t(X_t))$ in~\eqref{eq:acc-flow-comparison}. Because of this difference, the resulting distributions are different. Let $p_t(x,v)$ denote the joint distribution on $(X_t,v_t)$ of~\eqref{eq:under-dampled-Langevin} and let $q_t(x,v)$ denote the joint distribution on $(X_t,v_t)$ of~\eqref{eq:acc-flow-comparison}. Then the corresponding Fokker-Planck equations are: 
\begin{align*}
\frac{\partial p}{\partial t}(x,v)&= - \nabla_x \cdot(p_t(x,v)v) + \nabla_v \cdot(p_t(x,v)(\gamma v +\nabla f(x)))+\Delta_{v} p_t(x,v)\\
\frac{\partial q}{\partial t}(x,v)&=- \nabla_x \cdot(q_t(x,v)v) + \nabla_v \cdot(q_t(x,v)(\gamma v +\nabla f(x))) + \nabla_v \cdot (q_t(x,y)\nabla_x\log(\pr_t(x)))
\end{align*}
where $\pr_t(x)=\int q_t(x,v)\ud v$ is the marginal of $q_t(x,y)$ on $x$. The final term in the Fokker-Planck equations are clearly different. The joint distributions are different as well.

The situation is in contrast to the first order Langevin equation, where the stochastic term $\sqrt{2}\ud B_t$ and $-\nabla \log(\pr_t(X_t))$ are equivalent, in the sense that the resulting distributions have the same marginal distribution as a function of time. To illustrate this point, consider the following two forms of the Langevin equation:
\begin{align}
\ud X_t &= -\nabla f(X_t)\ud t + \sqrt{2}\ud B_t\label{eq:first-order-Langevin}\\
\ud X_t &= -\nabla f(X_t)\ud t - \nabla \log(\pr_t(X_t))\label{eq:first-order-Langevin-logp}
\end{align}
Let $p_t(x)$ denote the distribution of $X_t$ of~\eqref{eq:first-order-Langevin} and let $q_t(x)$ denote the distribution of $X_t$ of~\eqref{eq:first-order-Langevin-logp}. The corresponding Fokker-Planck equations are as follows
\begin{align*}
\frac{\partial p}{\partial t}(x)&= - \nabla \cdot(p_t(x)\nabla f(x)) + \Delta p_t(x)\\
\frac{\partial q}{\partial t}(x)&=- \nabla \cdot(q_t(x)\nabla f(x)) + \nabla \cdot (q_t(x)\nabla\log(\pr_t(x))) \\&=- \nabla \cdot(q_t(x)\nabla f(x)) + \nabla \cdot (q_t(x)\nabla\log(q_t(x))) \\&= - \nabla \cdot(q_t(x)\nabla f(x)) + \Delta q_t(x)
\end{align*}
where we used $\pr_t(x)=q_t(x)$. In particular, this implies that the marginal probability distribution of the stochastic process $X_t$ are the same for first order Langevin sde~\eqref{eq:first-order-Langevin} and~\eqref{eq:first-order-Langevin-logp} .

\end{document}